\documentclass{article} 
\pdfoutput=1
\usepackage{amsmath, amsthm, amssymb}

\usepackage{url}
\usepackage{color}
\usepackage{algorithm}
\usepackage{subfigure}
\usepackage{array}
\usepackage{xspace}
\usepackage{dsfont}
\usepackage[noend]{algorithmic}
\usepackage{multirow}
\newcommand{\remark}[1]{\textcolor{red}{\em #1\newline}}

\usepackage{tikz}
\usepgflibrary{arrows}
\usetikzlibrary{arrows}

\renewcommand{\epsilon}{\varepsilon}

\newtheorem{theorem}{Theorem}

\newtheorem{lemma}[theorem]{Lemma}

\newcommand{\ignore}[1]{}

\newcommand{\oneonegps}{(1+1)~GP*\xspace}

\newcommand{\oneonegpssingle}{(1+1)~GP*\nobreakdash-single\xspace}

\newcommand{\oneonegpsmulti}{(1+1)~GP*\nobreakdash-multi\xspace}

\newcommand{\Wlog}{W.\,l.\,o.\,g.\xspace}
\newcommand{\wrt}{w.\,r.\,t.\xspace}

\newcommand{\ie}{i.\,e.\xspace}
\newcommand{\eg}{e.\,g.\xspace}

\newcommand{\SORTING}{SORTING\xspace}
\newcommand{\HVLs}{HVL\nobreakdash-mutate'\xspace}

\usepackage[plainpages=false,pdfpagelabels]{hyperref}

\begin{document}

\title{Computational Complexity Results for Genetic Programming and the Sorting Problem}

\author{
\hspace*{-0mm}\mbox{Markus Wagner and Frank Neumann}\\
    School of Computer Science\\
    University of Adelaide\\
    5005 Adelaide, Australia
}

\date{}

\maketitle
\begin{abstract}
Genetic Programming (GP) has found various applications. Understanding this type of algorithm from a theoretical point of view is a challenging task. The first results on the computational complexity of GP have been obtained for problems with isolated program semantics. With this paper, we push forward the computational complexity analysis of GP on a problem with dependent program semantics. We study the well-known sorting problem in this context and analyze rigorously how GP can deal with different measures of sortedness.

\noindent\emph{ACM Category:} F.2, Theory of Computation, Analysis of Algorithms and Problem Complexity.
\end{abstract}



\section{Introduction}\label{sec:introduction}

Genetic programming (GP) \cite{koza:book} has proven to be very successful in various fields such as symbolic regression, financial trading, medicine, biology and bioinformatics (see e.g. Poli et al.~\cite{poli08:fieldguide}).
Various approaches such as schema theory, markov chain analysis, and approaches to measure problem difficulty have been used to understand GP from a theoretical point of view~\cite{PoliVLM10}.

Poli et al.~\cite{PoliVLM10} state, ``we expect to see computational complexity techniques being used to model simpler GP systems, perhaps GP systems based on mutation and stochastic hill-climbing.'' Computational complexity analysis has significantly increased the theoretical understanding of evolutionary algorithms for discrete search spaces. Here, one considers simplified versions of such algorithms and analyzes them rigorously on certain classes of problems by treating them as classical randomized algorithms~\cite{MotwaniRaghavan}. Taking this point of view, it allows one to use a sophisticated pool of techniques and to treat the algorithms in a strict mathematical sense.
Initial results on the computational complexity of evolutionary algorithms have been obtained for artificial pseudo-Boolean functions~\cite{RudBook,DJWoneone}. These results constitute the foundations for later results on classical combinatorial optimization, among them some of the most prominent problems in computer science such as minimum spanning trees, shortest paths, and maximum matchings (see Neumann and Witt~\cite{BookNeuWit} for an overview). 

Recently, the first computational complexity results for GP have been obtained by Durrett et al~\cite{GPOrderMajority2011}. In this paper, the authors consider simple GP algorithms on problems called ORDER and MAJORITY introduced by Goldberg and O'Reilly~\cite{goldberg:1998:good}. These two problems model isolated problem semantics and the analysis constitutes a first step towards obtaining deeper computational complexity results for GP.

Problems with isolated problem semantics are in a sense easy as they allow one to treat subproblems independently. The next step would be to consider problems that have dependent problem semantics and we follow this path in this paper.
Our goal is to push forward the computational complexity analysis of GP by examining a problem with dependent problem semantics, namely the sorting problem. Sorting problem is one of the most basic problems in computer science. It is also the first combinatorial optimization problem for which computational complexity results have been obtained in the area of discrete evolutionary algorithms~\cite{EASorting2004,EADirected2008}. In~\cite{EASorting2004}, sorting is treated as an optimization problem where the task is to minimize the unsortness of a given permutation of the input elements. To measure unsortness, different fitness functions have been introduced and studied with respect to the difficulty of being optimized by permutation-based evolutionary algorithms.

We consider the simple GP algorithms set up in \cite{GPOrderMajority2011} and analyze them on the different fitness functions of the sorting problem proposed in~\cite{EASorting2004}. 
Our analyses point out how GP algorithms can deal with this problem that has dependent problem semantics and provide rigorous insights into the optimization process of our GP systems.
As classical GP systems work on tree-based structures and allow many different solutions to a given problem, our investigations have to be significantly different from the ones carried out in \cite{EASorting2004}. One crucial difference is that elements may occur more than once in a tree. This leads for some of the fitness functions to local optima and prevents our GP algorithms from obtaining an optimal solution in expected polynomial time.

The outline of the paper is as follows. In Section~\ref{sec:definitions}, we introduce the algorithms that are subject to our analysis and present our model of the sorting problem. Section~\ref{sec:lowerBounds} presents lower bounds on the expected optimization time, and Section~\ref{sec:upper} presents upper bounds for sortedness measures that lead to an efficient optimization process.
Worst case situations and lower bounds are presented for sortedness measures which may make the algorithms getting stuck in Section~\ref{sec:worst}. Finally, we finish with some concluding remarks.

\section{Definitions}\label{sec:definitions}

%

\subsection{Program Initialization}

When considering tree-based genetic programming, a set of primitives $A$ has to be selected, where $A$ contains a set $F$ of functions and a set $L$ of terminals. 
The semantics of each primitive is explicitly defined. For example, a primitive might represent the value bound to an input variable, an arithmetic operation, or a branching statement such as an IF-THEN-ELSE conditional. 
Functions are parameterized, and terminals are either functions with no parameters, i.e.~arity equal to zero, or input variables to the program that serve as actual parameters to the formal parameters of functions.

For our investigations, we assume that a GP program is initialized in the following way: the root node is randomly drawn from $A$, and subsequently, the parameters of each function are recursively populated with random samples from $A$, until the leaves of the tree are all terminals. Thus, functions constitute the internal nodes of the parse tree, and terminals occupy the leaf nodes. 

\subsection{HVL-mutate'}

The \HVLs operator is an update of the HVL mutation operator (\cite{OReilly:thesis}) and is motivated by minimality. The original HLV first selects a node at random in a copy of the current parse tree. Let us term this the \texttt{currentNode}. It then, with equiprobability, applies one of three sub-operations: insertion, substitution, or deletion. Insertion takes place above \texttt{currentNode}: a randomly drawn function from $F$ becomes the parent of \texttt{currentNode} and its additional parameters are set by drawing randomly from $L$. Substitution changes \texttt{currentNode} to a randomly drawn function of $F$ with the same arity. Deletion replaces \texttt{currentNode} with its largest child subtree, which often admits large deletion sub-operations.

The variation of HLV that we consider here functions slightly differently, since we restrict it to operate on trees where all functions take two parameters. Rather than choosing a node followed by an operation, we first choose one of the three sub-operations to perform. Then, the operations proceed as shown in Figure~\ref{fig:hvl_prime_example}. Insertion and substitution are exactly as in HVL; however, deletion only deletes a leaf and its parent to avoid the potentially macroscopic deletion change of HVL that is not in the spirit of bit-flip mutation. This change makes the algorithm more amenable to complexity analysis and specifies an operator that is only as general as our simplified problems require, contrasting with the generality of HVL, where all sub-operations handle primitives of any arity. Nevertheless, both operators respect the nature of GP's search among variable-length candidate solutions because each generates another candidate of potentially different size, structure, and composition.


\begin{figure}[htb]
\centering
\subfigure[Before insertion]
  {\includegraphics[width=1.3in,height=0.95in,trim=10mm 15mm 0mm 0mm]{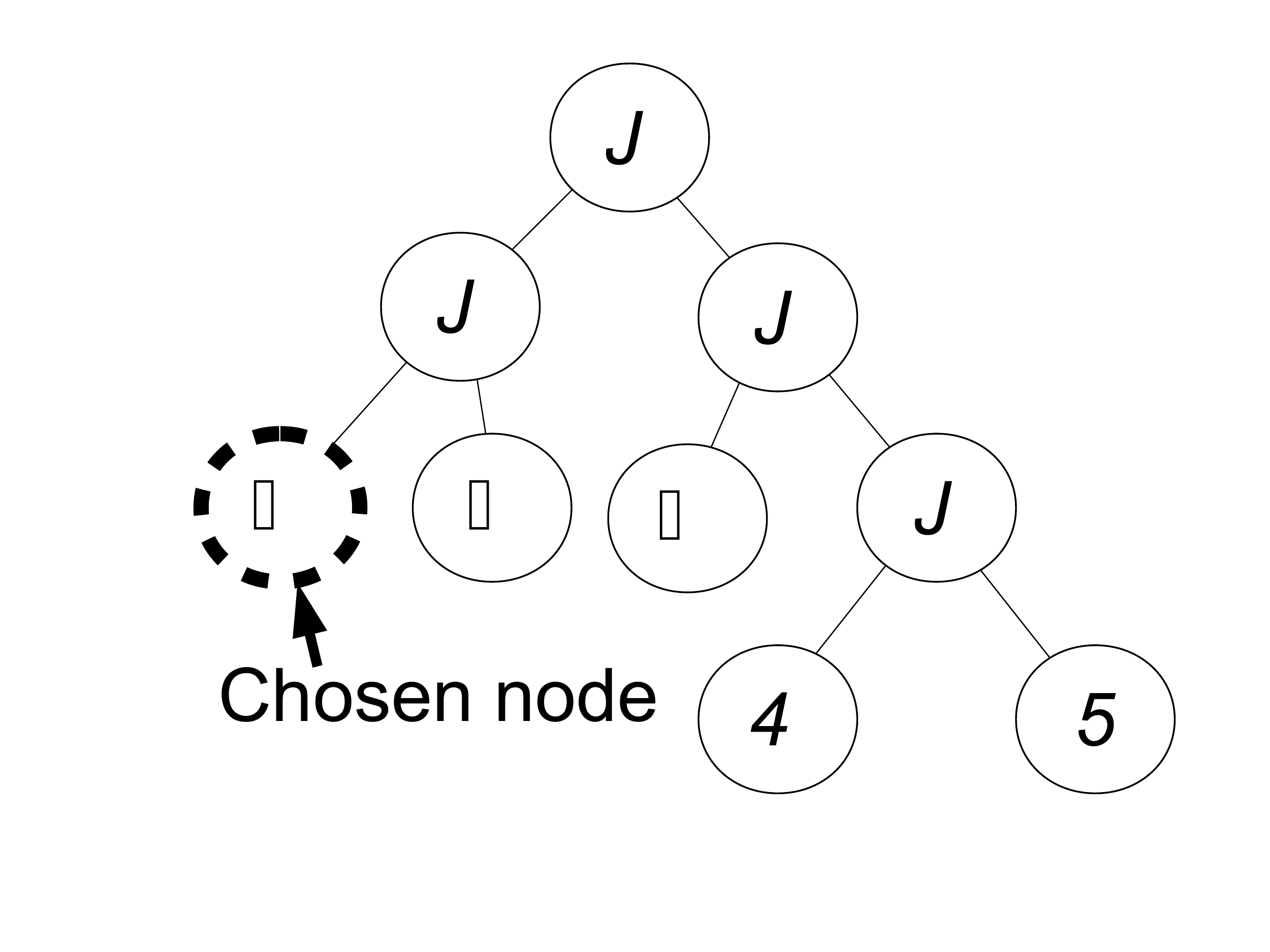}}
\hspace{0.2in}
\subfigure[After insertion]
  {\includegraphics[width=1.4in,height=0.95in,trim=10mm 15mm 0mm 0mm]{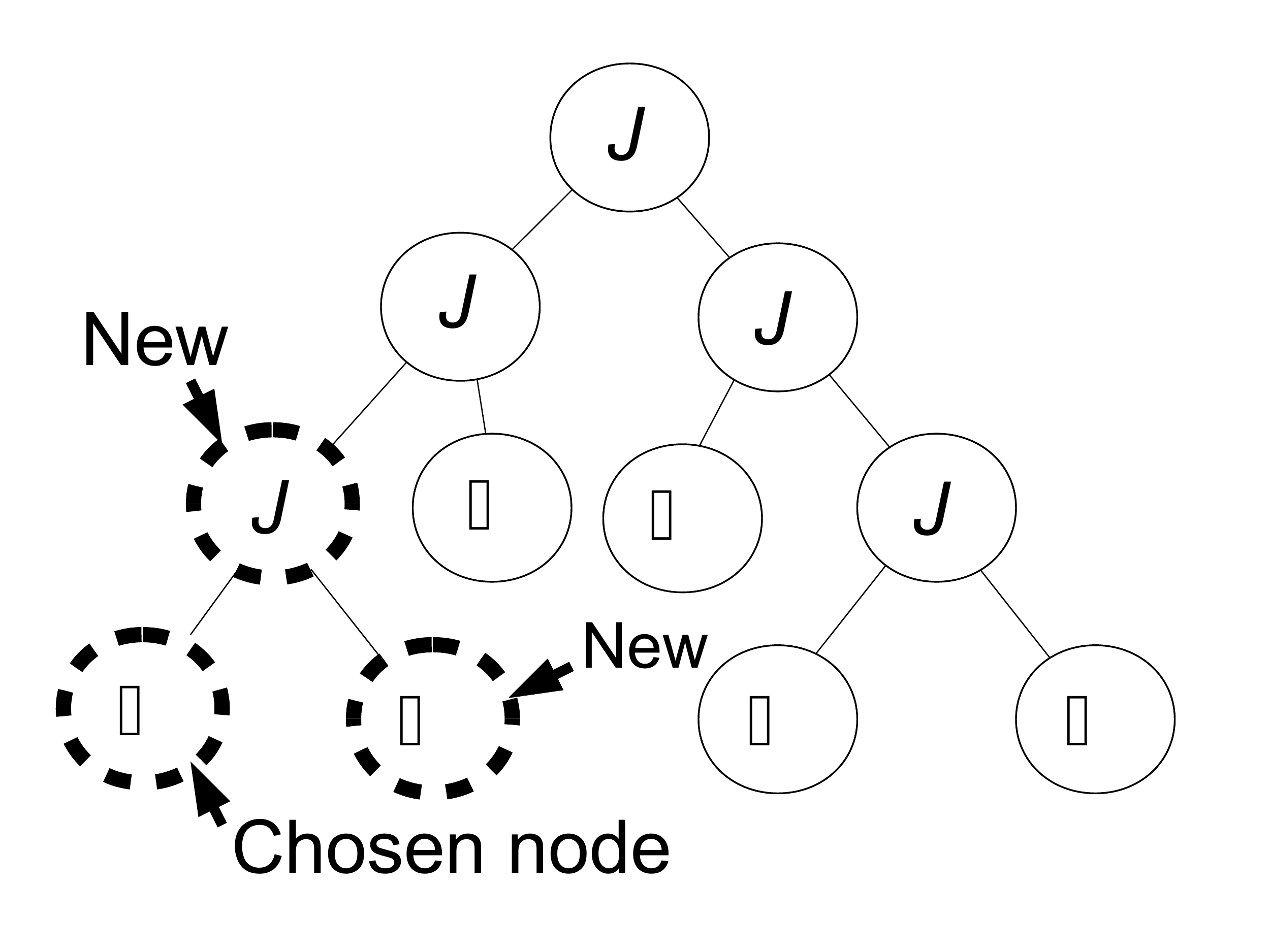}}
\subfigure[Before deletion]
  {\includegraphics[width=1.3in,height=0.95in,trim=10mm 15mm 0mm 0mm]{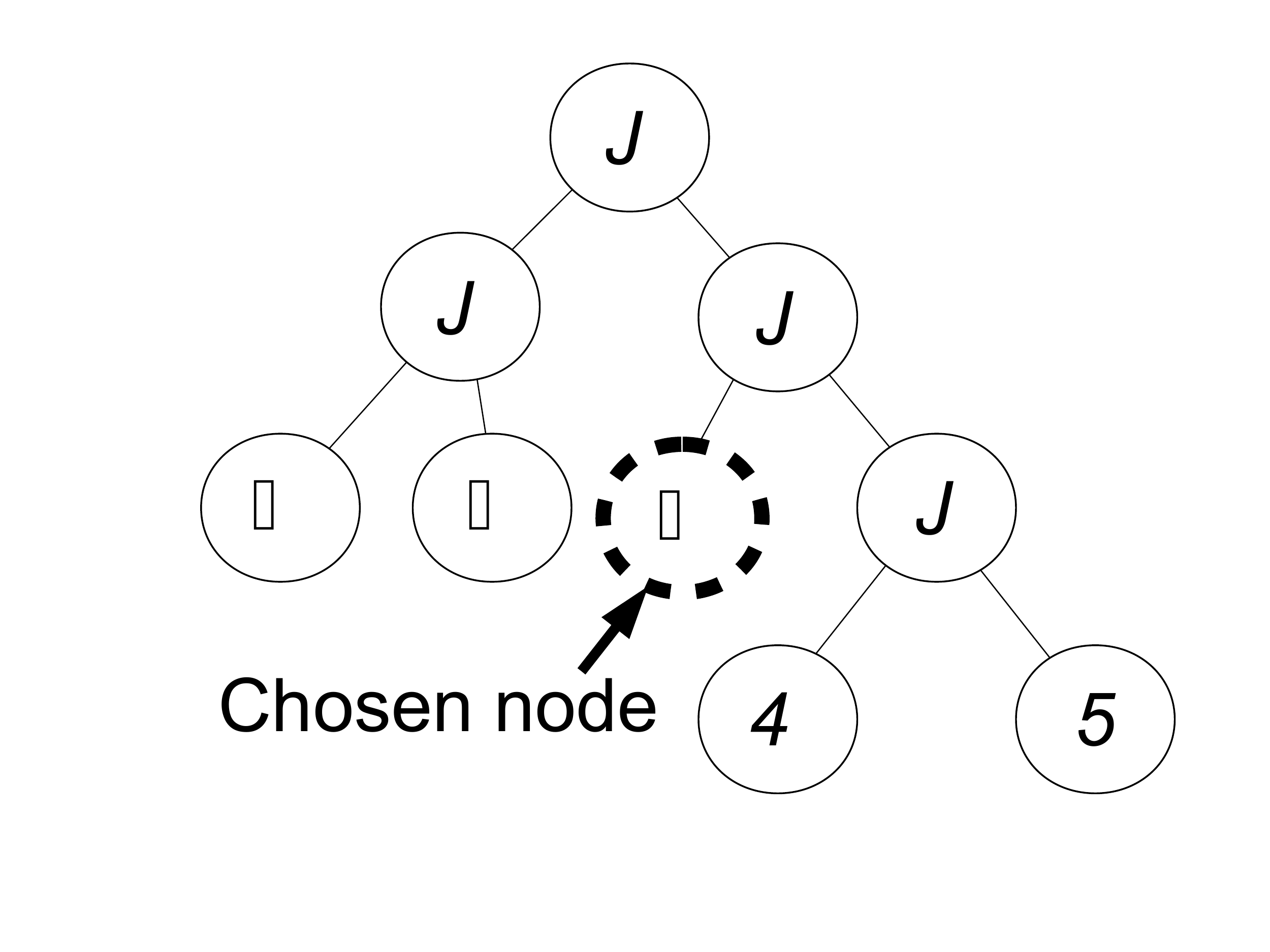}}
\hspace{0.2in}
\subfigure[After deletion]
  {\includegraphics[width=1.4in,height=0.95in,trim=10mm 30mm 0mm 0mm]{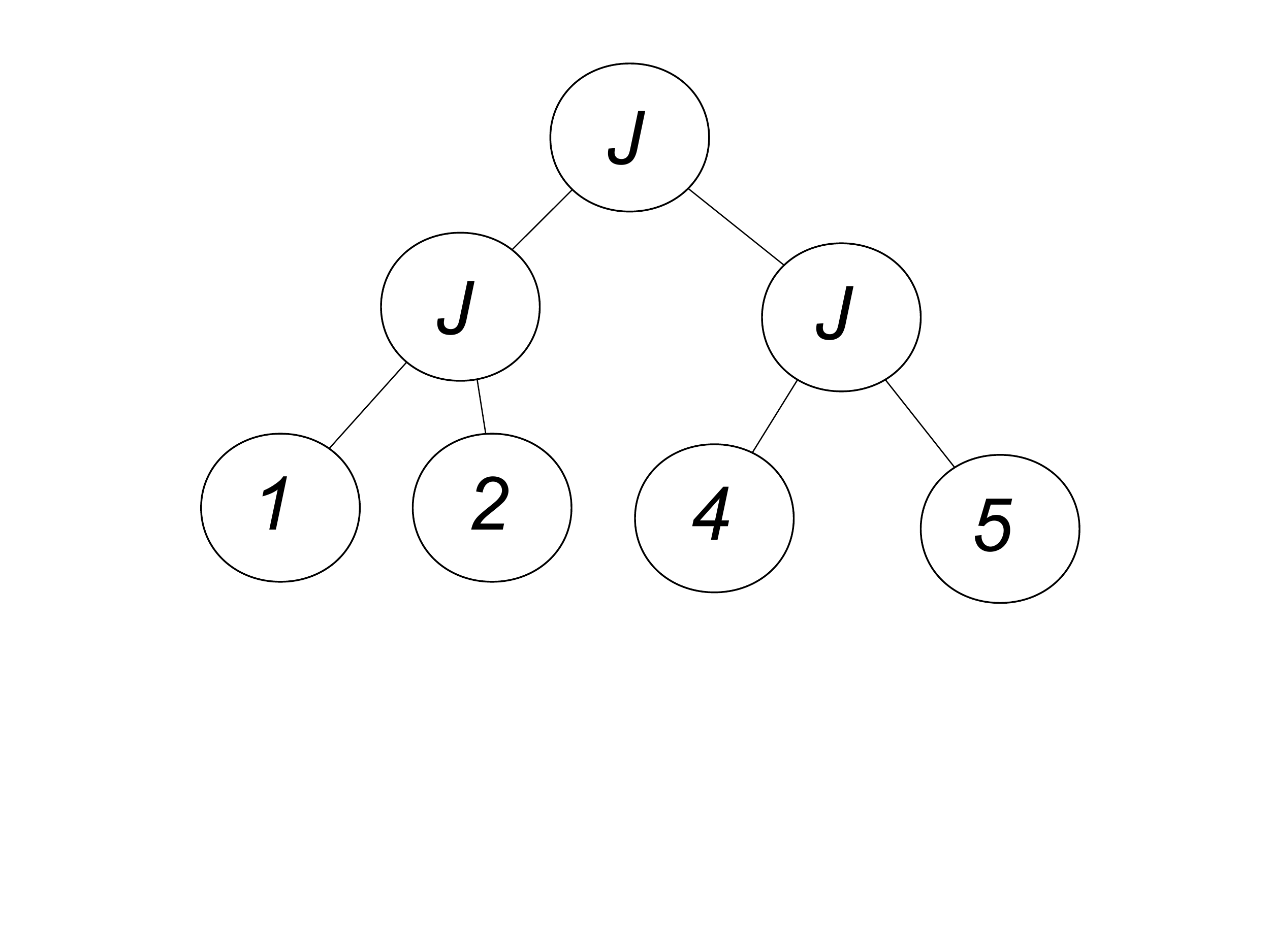}}
\subfigure[Before substitution]
  {\includegraphics[width=1.3in,height=0.95in,trim=10mm 15mm 0mm 0mm]{current-hvl-mutate-prime-2.pdf}}
\hspace{0.2in}
\subfigure[After substitution]
  {\includegraphics[width=1.3in,height=0.95in,trim=10mm 15mm 0mm 0mm]{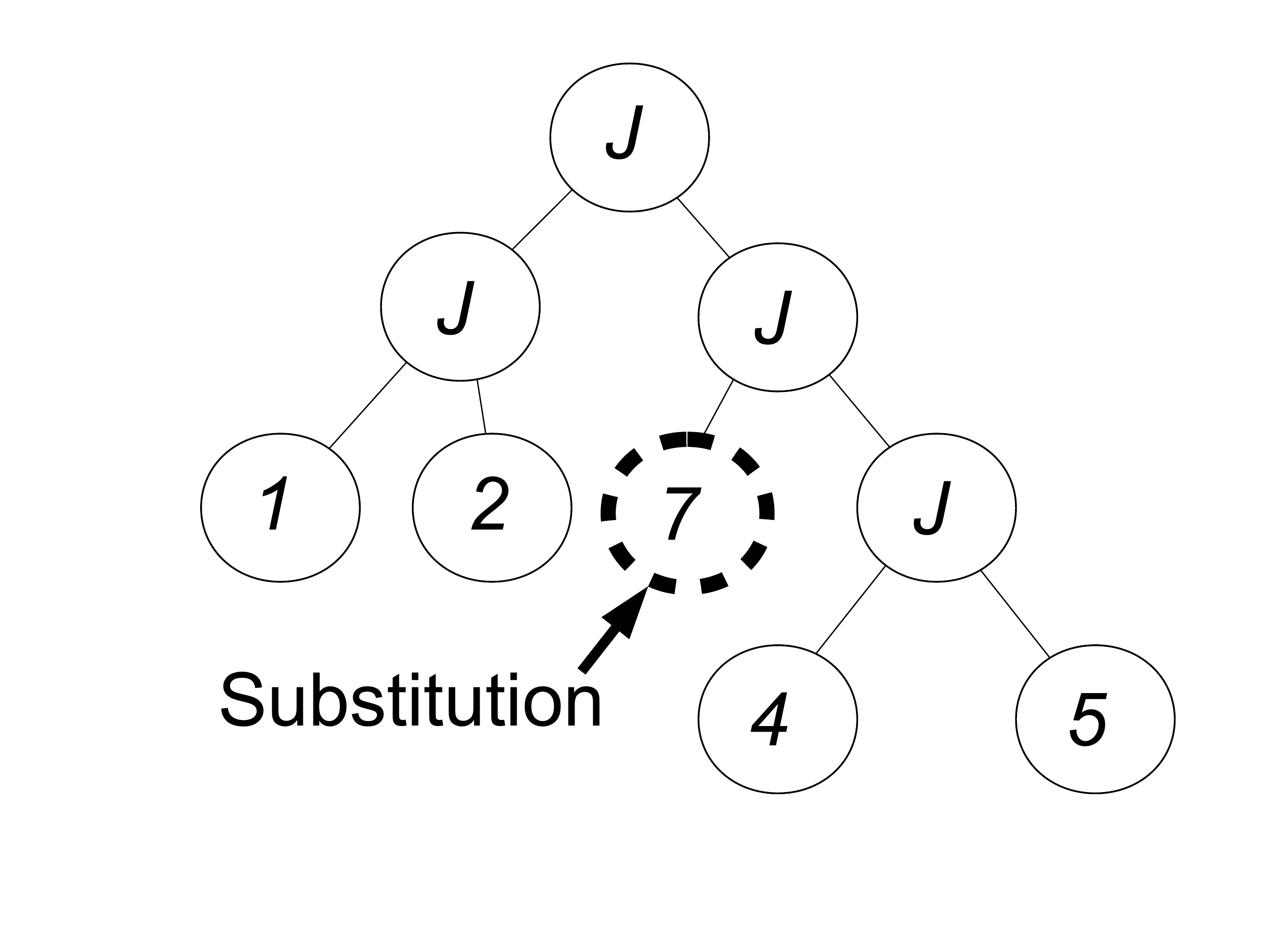}}
\caption{Example of the operators from \HVLs.}
\label{fig:hvl_prime_example}
\end{figure}

\subsection{Algorithms}\label{sec:algGP}

We define the genetic programming variant called\linebreak \oneonegps. It works with a population of size one and produces in each iteration one single offspring. \oneonegps is defined in Algorithm~\ref{gp} and accepts an offspring only if it is strictly fitter than its parent.

\begin{algorithm}[h!]
\caption{\oneonegps}
\begin{enumerate}
\item Choose an initial solution $X$.
\item Set $X':= X$.
\item Mutate $X'$ by applying \HVLs $k$ times. For each application, randomly choose to either substitute, insert, or delete.
\begin{itemize}
\item If substitute, replace a randomly chosen leaf of $X'$ with a new leaf $u \in L$ selected uniformly at random.
\item If insert, randomly choose a node $v$ in $X'$ and select $u \in L$ uniformly at random. Replace $v$ with a join node whose children are $u$ and $v$, with the order of the children chosen randomly.
\item If delete, randomly choose a leaf node $v$ of $X'$, with parent $p$ and sibling $u$. Replace $p$ with $u$ and delete $p$ and $v$.
\end{itemize}
\item If $f(X') > f(X)$, set $X:=X'$.
\item Go to 2.
\end{enumerate}
\label{gp}
\end{algorithm}

Additionally, we consider a variant of \oneonegps which potentially applies \HVLs more then once when a child is generated. 
Thus, for\linebreak\oneonegpssingle, we set the number of applications to $k=1$, so that we perform one mutation at a time according to the \HVLs framework, and for \oneonegpsmulti, we choose $k = 1+\text{Pois}(1)$, so that the number of mutations at a time varies randomly according to the Poisson distribution.

We will analyze these two algorithms in terms of the expected number of fitness evaluations that is needed to produce an optimal solution for the first time. This is called the expected optimization time of the algorithm.

\subsection{The \SORTING Problem}

Given a set of $n$ elements from a totally ordered set, sorting is the problem of ordering these elements. We will identify the given elements by $1, \ldots, n$.

The goal is to find a permutation $\pi_{opt}$ of $1, \ldots, n$ such that 
\[
\pi_{opt}(1) < \pi_{opt}(2) <  \ldots < \pi_{opt}(n)
\]
holds, where $<$ is the order on the totally ordered set. \Wlog we assume $\pi_{opt}=id$, \ie $\pi_{opt}(i)=i$ for all $i$, throughout this paper.

The set of all permutations $\pi$ of $1, \ldots, n$ forms a search space that has already been investigated in \cite{EASorting2004} for the analysis of permutation-based evolutionary algorithms.  The authors of this paper, investigate sorting as an optimization problem whose goal is to maximize the sortedness of a given permutation. 
The following  fitness functions measuring the sortedness of a given permutation $\pi$  have been analyzed in \cite{EASorting2004} .

\begin{itemize}
	\item $INV(\pi)$, measuring the number of pairs in correct order,\footnote{Originally, $INV$ measures the numbers of pairs in wrong order. Our interpretation has the advantage that we need no special treatment of incompletely defined permutations.} which is the number of pairs $(i,j)$, $1 \leq i < j \leq n$, such that $\pi(i) < \pi(j)$,
	\item $HAM(\pi)$, measuring the number of elements at correct position, which is the number indices $i$ such that $\pi(i)=i$,
	\item $RUN(\pi)$, measuring the number of maximal sorted blocks, which is the number of indices $i$ such that $\pi(i+1) < \pi(i)$ plus one,
	\item $LAS(\pi)$, measuring the length of the longest ascending subsequence, which is the largest $k$ such that $\pi(i_1) < \ldots < \pi(i_k)$ for some $i_1 < \ldots < i_k$,
	\item $EXC(\pi)$, measuring the minimal number of pairwise exchanges in $\pi$, in order to sort the sequence.
\end{itemize}

Note that $EXC(\pi)$ can be computed in linear time, based on the cycle structure of permutations. If the sequence is sorted, it has $n$ cycles. Otherwise, it is always possible to increase the number of cycles by exchanging an element that is not sitting at its correct position with the element that is currently sitting there. For any given permutation $\pi$ consisting of $n-k$ cycles, $EXC(\pi)=k$.


We want to investigate sorting in the context of genetic programming. Note, that the fitness functions encounter several interactions between the elements of the permutation. Initial investigations on the computational complexity analysis of genetic programming considered isolated problem semantics~\cite{GPOrderMajority2011} and an important step is to investigate what happens if dependencies are involved. Therefore, the sorting problem modeled as an optimization problem seems to be ideal to get further rigorous insight into the optimization behavior of genetic programming.

Considering tree-based genetic programming, we have to deal with the fact that certain elements are not present in a current tree. 
We extend our notation of permutation to incompletely defined permutations.
Therefore, we use $\pi$ to denote a list of elements, where each element of the input set occurs at most once. 
This is a permutation of the elements that occur in the tree.
Furthermore, we use $\pi(x)=p$ to get the position $p$ that the element $x$ has within $\pi$.
In the case that $x \notin \pi$, $\pi(x)=\bot$ holds. 
We adjust the definition of $\pi$ to later accommodate the use of trees as the underlying data structure. 
For example, $\pi=(1,2,4,6,3)$ leads to $\pi(1)=1$, $\pi(2)=2$, $\pi(3)=5$, $\pi(4)=3$, $\pi(6)=4$, and $\pi(5)=\bot$.

\ignore{
To guide the search process, we have to specify an appropriate fitness function as the measure for sortedness. A large number of such measures exists and many have been subject to studies (for example \cite{sortednessStudiesPetersson95}). 
We will consider the following fitness functions introduced in \cite{EASorting2004} which measure sortness in different ways.
}

In order to deal with incompletely defined permutations, we need to complete the measures that are to be minimized, namely $RUN$ and $EXC$. We assign a fitness of $n+1$ to incompletely defined permutations. 

\ignore{
\remark{[comment for the future] quite strong, but it does not matter for the non-*-variants\\
An alternative: minimize $-|\pi| + RUN(\pi)$ and $-|\pi| + EXC(\pi)$ to prevent large plateaus}
}

The set of primitives used in our GP-variants is the union of the following two sets:

\begin{itemize}
	\item $F:=\{J\}$, $J$ has arity 2,
	\item $L:=\{1, \ldots, n\}$.
\end{itemize}

Algorithm~\ref{alg:sorting} describes how the fitness of a tree is computed.

\begin{algorithm}[h!]
\caption{Derivation of $f(X)$ for \SORTING}\label{alg:sorting}
\begin{enumerate}
\item Derive a possibly incompletely defined permutation $P$ of X:
\begin{enumerate}
\item[Init: ]$l$ an empty leaf list, $P$ an empty list representing a possibly incompletely defined permutation
\item[1.1] Parse $X$ in order and insert each leaf at the rear of $l$ as it is visited.
\item[1.2] Generate $P$ by parsing $l$ front to rear and adding (``expressing'') a leaf to $P$ only if it is not yet in $P$, \ie it has not yet been expressed.
\end{enumerate}
\item Compute $f(X)$ based on $P$ and the chosen fitness function.
\end{enumerate}
\end{algorithm}

For example, for a tree $X$ with (after the in order parse) $l=(2,2,3,4,5,1,6,3)$ and $\left|L\right|=6$, $P=(2,3,4,5,1,6)$, the sortedness results are $INV(P)=10$, $HAM(P)=1$, $RUN(P)=2$, $LAS(P)=4$, and $EXC(P)=4$.


\section{General Lower Bounds}\label{sec:lowerBounds}

We start with a simple lower bound, that is independent of the used sortedness measure.

\begin{theorem}\label{thm:lowerBound}
Starting with a non-optimal solution, the expected optimization time of the single- and multi-operation cases of \oneonegps on \SORTING is $\Omega(n^2)$ if the deletion of nodes is not allowed, and $\Omega(n)$ else, where $n=\left|L\right|$. 
\end{theorem}

In order to prove this, we have to bound the number of different mutations that can lead to an optimal tree. This is done in the following lemma.

\begin{lemma}\label{lem:lowerBoundImprovement}
For any given non-optimal tree $X$ and its in order parsed list of leaves $l$, there exist at most three different sub-operations of \HVLs that can change $X$ into an optimal tree.
\end{lemma}

\begin{proof}
The proof is done by investigating the different cases of near-optimal individuals that can be improved to the optimal one in a single mutation. In the following, 
we denote by $x\_\_x$ a sequence of leaves labeled $x$.
\begin{enumerate}
\item[Case 1] An element $i \in L$ is missing in $l$. 
\begin{itemize}
	\item If $i=1$ and $l=(2\_\_2, \ldots, n)$, then an insertion of $1$ at position $1$ results in an optimal tree. Furthermore, a substitution of $2$ to $1$ results in another optimal tree.
	\item If $1<i<n$ and $l=( \ldots, \mbox{$i$-1}\_\_\mbox{$i$-1}, \mbox{$i$+1}\_\_\mbox{$i$+1}, \ldots)$, then an insertion of $i$ between the \mbox{$i$-1}'s and the \mbox{$i$+1}'s results in an optimal tree. Alternatively, substitutions of the rightmost \mbox{$i$-1} or of the leftmost \mbox{$i$+1} yield further optimums.
	\item If $i=n$ and $l=(\ldots, \mbox{$n$-1}\_\_\mbox{$n$-1}, n\_\_n)$, then an insertion of $n$ at the rightmost position, or a substitution of the rightmost \mbox{$n$-1} yield optimal trees.
\end{itemize}
\item[Case 2] An element $x \in L$ is at an incorrect position $p$ in $l$, thus possibly preventing other $x$ in the rest of the list from becoming expressed.
\begin{itemize}
	\item If $p=1$ and $l=(x, 1\_\_1, \ldots)$, it is possible to delete $i$, or to substitute $i$ by a leaf labeled 1, resulting in optimal trees. 
	\item If $1<x<=n$ and $l=(\ldots, \mbox{$i$-1}\_\_\mbox{$i$-1}, x, i\_\_i, \ldots)$, then it is possible to delete $x$, or to substitute $x$ by \mbox{$i-1$} or by $i$.
\end{itemize}
\end{enumerate}
\end{proof}

Note that the investigated individuals represent maximal cases \wrt the number of possible optimizing mutations. For example, for a tree $X$ with  $l=(2,3,\ldots)$, an exchange of the $2$ to a $1$ would obviously not yield an optimal tree.

Now, it is possible to prove Theorem~\ref{thm:lowerBound}.

\begin{proof}[Proof of Theorem~\ref{thm:lowerBound}]
We investigate the final step producing the optimal individual. 
There, it is necessary, that the last application of an \HVLs sub-operation produces the optimal individual. 
For the single-mutation variant, the tree size at this stage is at least $n-1=\Omega(n)$. For the multi-mutation variant, the size is at least $n-\sqrt{n}=\Omega(n)$ with high probability, as the probability to perform more than $\sqrt{n}$ operations is $e^{-\Omega(\sqrt{n})}$.

Based on Lemma~\ref{lem:lowerBoundImprovement}, for each non-optimal individual, there are at most a total of three sub-operations to change it into the optimal one. 
For the sub-operation insertion, the probability of success, \ie for inserting the needed terminal at the correct position, is bounded above by $\frac{1}{3} \frac{1}{n \Omega(n)}$. 
Similarly, the success probability for substitution\footnote{Note that in some cases, two different substitutions may result in the optimal solution.} is bounded above by $\frac{1}{3} \frac{2}{n \Omega(n)}$, and for a deletion by $\frac{1}{3} \frac{1}{\Omega(n)}$. 
Hence, the probability of a success is bounded above by 
$$\frac{1}{n \Omega(n)}+\frac{2}{n \Omega(n)}=O\left(\frac{1}{n^2}\right)$$
 if no deletion of nodes is allowed, and by 
$$\frac{1}{n \Omega(n)}+\frac{2}{n \Omega(n)}+\frac{1}{\Omega(n)}=O\left(\frac{1}{n}\right)$$ 
 else. Thus the waiting times for single sub-operations are bounded from below by $\Omega(n^2)$ and $\Omega(n)$.
\end{proof}

\section{Upper Bound}
\label{sec:upper}



In this section we analyze the performance of our GP variants on one of the fitness functions introduced in Section~\ref{sec:definitions}. 

We exploit a similarity between our variants and evolutionary algorithms to obtain an upper bound. 
We use the method of fitness-based partitions, also called fitness-level method, to estimate the expected optimization time. This method has originally been introduced for the analysis of elitist evolutionary algorithms (see, \eg, Wegener~\cite{Wegener2002}) where the fitness of the current search point can never decrease. The idea is to partition the search space into levels $A_1, \dots, A_m$ that are ordered with respect to fitness values. Formally, we require that for all $1 \le i \le m-1$ all search points in $A_i$ have a strictly lower fitness than all search points in $A_{i+1}$. In addition, $A_m$ must contain all global optima.

Now, if $s_i$ is (a lower bound on) the probability of discovering a new search point in $A_{i+1} \cup \dots \cup A_m$, given that the current best solution is in $A_i$, the expected optimization time is bounded by $\sum_{i=1}^{m-1} 1/s_i$, as $1/s_i$ is (an upper bound on) the expected time until fitness level~$i$ is left and each fitness level has to be left at most once.

Although the used \HVLs operator is complex, we can obtain a lower bound on the probability of making an improvement considering fitness improvements that arise from the \HVLs sub-operations insertion and substitution. In combination with fitness levels defined individually for the used sortedness measures, this gives us the runtime bounds in this section. 

Let us denote by $T_{max}$ the maximal tree size at any stage during the evolution of the algorithm, and by $T_k$ the tree size when the fitness $k$ is achieved during a run.

\begin{theorem}\label{thm:INVupperBound}
The expected optimization time of the single- and multi-operation cases of \oneonegps with $INV$ is $O(n^3T_{max} )$.
\end{theorem}

\begin{proof}

The proof is an application of the above-described fitness-based partitions method. 
Based on the observation that $n\cdot(n-1)/2+1$ different fitness values are possible, 
we define the fitness levels $A_0, \ldots, A_{n\cdot(n-1)/2}$ with 
\[
A_i= \left\{ \pi \left|  INV(\pi)=i  \right.\right\} .
\]
As there are at most $n\cdot(n-1)/2$ advancing steps between fitness levels to be made, the total runtime is bounded by the sum over all times needed to make such steps. 

We bound the times by investigating the case, when only a particular insertion of a specific leaf at its correct position achieves an increase of the fitness.\footnote{Examplarily, the tree with $l=( \mbox{$n$}\_\_\mbox{$n$}, $1$, $2$, \ldots, \mbox{$n$-1})$ can only be improved (in a single step) by inserting a leaf labelled $1$ at the leftmost position.} 
The probability for such an improvement for \oneonegpssingle is $p_k=\Omega\left(\frac{1}{n T_k}\right)$. For \oneonegpsmulti, the probability for a single mutation operation occurring (including the mandatory one) is $1/e$; thus $p_k=\Omega\left(\frac{1}{n T_k}\right)$ in the multi-operation case as well.

Therefore, 
the total optimization time is 
\begin{eqnarray*}
\sum_{k=0}^{n\cdot(n-1)/2} O\left(nT_{max}\right)  
=O(n^3T_{max} ).
\end{eqnarray*}
\end{proof}

\renewcommand{\arraystretch}{1.5}
\begin{table*}
	\centering
		\begin{tabular}{c|c|c}
\multirow{2}{*}{\renewcommand\arraystretch{0.95}\begin{tabular}[b]{@{}l@{}}Fitness\\function\end{tabular}} & \multicolumn{2}{c}{\oneonegps} \\ 
 & single & multi \\ \hline
INV & $O(n^3T_{max} )$ & $O(n^3T_{max} )$\\
HAM & $\infty$ & $\Omega \left( \left( \frac{n}{e} \right)^{n}  \right)$ \\
RUN & $\infty$ & $\Omega \left( \left( \frac{n}{e} \right)^{n}  \right)$ \\
LAS & $\infty$ & $\Omega \left( \left( \frac{n}{e} \right)^{n}  \right)$ \\
EXC & $\infty$ & $\Omega \left( \left( \frac{n}{e} \right)^{n}  \right)$ \\
\end{tabular}
\caption{Summary of results.
Note that unless another lower bound is given, the lower bounds of $\Omega(n^2)$ and $\Omega(n)$ from Section~\ref{sec:lowerBounds} hold. 
}
\label{tab:results}
\end{table*}

\section{Worst Case Situations}
\label{sec:worst}

In the following, we examine our algorithms for the remaining measures of sortedness. 
We present several worst case examples for $HAM$, $RUN$, $LAS$, and $EXC$ that demonstrate that \oneonegpssingle and \oneonegpsmulti can get stuck during the optimization process. 
This shows that evolving our GP system is much harder than working with the permutation-based EA presented in where only the sortedness measure $RUN$ leads to an exponential optimization time.

We restrict ourselves to the case where we initialize with a tree of size linear in $n$ and show that even this leads to difficulties for the mentioned sortedness measures. Note, that a linear size is necessary to represent a complete permutation of the given input elements. 

For $RUN$ and $LAS$, we investigate the following initial solution called $T_{w1}$ and show that it is hard for our algorithms to achieve an improvement.

$$ \underbrace{n, n, \ldots, n}_{n+1 \text{ of these}}, 1, 2, 3, \ldots, n $$

\begin{theorem}\label{thm:lowerBoundWorstCaseRUNLASsingle}
Let $T_{w1}$ be the initial solution to \SORTING. Then the expected optimization time of \oneonegpssingle and  \oneonegpsmulti is infinite respectively $e^{\Omega(n)}$ for the sortedness measures $RUN$ and $LAS$.
\end{theorem}

\begin{proof}
We consider \oneonegpssingle first.
It is clear that, with a single\linebreak\HVLs application, only one of the leftmost $n$s can be removed. For an improvement in the sortedness based on $RUN$ or $LAS$, all leftmost $n+1$ leaves have to be removed at once. This cannot be done by the \oneonegpssingle, resulting in an infinite runtime.

\oneonegpsmulti can only improve the fitness is by removing the leftmost $n+1$ leaves. 
Hence, in order to successfully improve the fitness, at least $n+1$ sub-operations have to be performed, assuming that we, in each case, delete one of the leftmost $n$s. Because the number of sub-operations per mutation is distributed as $1+Pois(1)$, the Poisson random variable has to take a value of at least $n$. 
This implies that the probability for such a step is $e^{-\Omega(n)}$ and the expected waiting time for such a step is therefore $e^{\Omega(n)}$ which completes the proof.
\end{proof}

Similarly, we consider the tree $T_{w2}$ which has as leaves the elements

$$ \underbrace{n, n, \ldots, n}_{n+1 \text{ of these}}, 2, 3, \ldots, n-1, 1, n $$

and show that this is hard to improve when using the sortedness measures $HAM$ and $EXC$.

\begin{theorem}\label{thm:lowerBoundWorstCaseEXCsingle}
Let $T_{w2}$ be the initial solution to \SORTING. Then the expected optimization time of \oneonegpssingle and  \oneonegpsmulti is infinite respectively $e^{\Omega(n)}$ for the sortedness measures $HAM$ and $EXC$.
\end{theorem}

\begin{proof}
We use similar ideas as in the previous proof.
Again, it is not possible for \oneonegpssingle to improve the fitness in a single step, as all $n+1$ leftmost leaves have to be removed in order for the rightmost $n$ to become expressed. Additionally, a leaf labeled $1$ has to be inserted at the beginning, or alternatively, one of the $n+1$ leaves labeled $n$ has to be replaced by a $1$. 
This results in a minimum number of $n+1$ sub-operations that have to be performed by a single \HVLs application, leading to the lower bound of $e^{\Omega(n)}$ for \oneonegpsmulti.
\end{proof}


\section{Conclusions}\label{sec:conclusions}

Genetic programming is successfully applied in numerous fields. However, its computational complexity analysis has just been started recently. Thus far, only problems with independent problem semantics have been analyzed. We investigated a first problem with dependent semantics, namely the sorting problem.
Analyzing the set up of of Durrett et al~\cite{GPOrderMajority2011} together with the fitness measures proposed by Scharnow et al. \cite{EASorting2004}, we have shown how the algorithms behave on different measures on sortedness.

Our results are summarized in Table~\ref{tab:results}.
For the measure $INV$ we have presented polynomial bounds on the expected optimization time. For the remaining measurements $HAM$, $RUN$, $LAS$, and $EXC$, we have pointed out situations where the algorithms get stuck. 
Our analyses give further rigorous insights into the behavior of simple GP systems. Furthermore, it shows the fact that if multiple occurrences of variables are allowed in the system, this may make the optimization task hard much harder than for permutation-based evolutionary algorithms, where only single occurrences are allowed.

\ignore{
\remark{Differences to Scharnow-results: Scharnow uses complete permutations and  the operators \texttt{exchange} and \texttt{jump}, which both can have a larger influence on the resulting permutation than \HVLs.}

In the future, we plan to extend our work, in order to further push the complexity analysis of genetic programming. 
By using the tuple $\left\langle \mbox{sortedness measure}, T_{max}\right\rangle$ (with the number of leaves as the second component) as fitness values, and then using a lexicographic ordering, this may allow the $1+1$-variants to solve more problems independent of the used sortedness measure.
Furthermore, the extension to the $\mu+1$ variants is possible, and to more problems with dependent problem semantics.
}

\bibliographystyle{abbrv}
\bibliography{references} 
\end{document}